\documentclass[letterpaper, 10 pt, conference]{ieeeconf}  %

\IEEEoverridecommandlockouts                              %

\usepackage{url}
\usepackage[hidelinks]{hyperref}

\usepackage[
  backend=biber,
  style=numeric-comp,   %
  sortcites=true,
  sorting=none,         %
  natbib=true,          %
  giveninits=true,
  maxcitenames=4,       %
  minbibnames=3,        %
  maxbibnames=4,        %
  doi=false,            %
  url=false,
  hyperref=true
]{biblatex}
\DeclareSourcemap{
  \maps[datatype=bibtex]{
    \map{
      \step[fieldset=file,      null]
      \step[fieldset=abstract,  null]
      \step[fieldset=keywords,  null]
      \step[fieldset=urldate,   null]
      \step[fieldset=month,     null]
      \step[fieldset=day,       null]
      \step[fieldset=language,  null]
      \step[fieldset=isbn,      null]
      \step[fieldset=issn,      null]
      \step[fieldsource=doi,
            match=\regexp{^\s*https?://(dx\.)?doi\.org/},
            replace={}]
      \step[fieldsource=booktitle,
            match=\regexp{^\s*Proceedings\s+of\s+(the\s+)?},
            replace={}]
      \step[fieldsource=eventtitle,
            match=\regexp{^\s*Proceedings\s+of\s+(the\s+)?},
            replace={}]
      \step[fieldsource=booktitle,
            match=\regexp{^\s*(19|20)\d{2}\s+},
            replace={}]
      \step[fieldsource=eventtitle,
            match=\regexp{^\s*(19|20)\d{2}\s+},
            replace={}]
      \step[fieldsource=booktitle,
            match=\regexp{^\s*\d+(st|nd|rd|th)\s+},
            replace={}]
      \step[fieldsource=eventtitle,
            match=\regexp{^\s*\d+(st|nd|rd|th)\s+},
            replace={}]
    }
  }
}
\renewbibmacro{in:}{}

\DeclareFieldFormat{title}{\mkbibemph{#1}}
\DeclareFieldFormat[article,incollection,techreport,inproceedings,book,misc]{title}{\mkbibquote{#1}}

\DeclareFieldFormat{eprint:arxiv}{arXiv\addcolon\space#1}


\usepackage{amsthm}
\newtheorem{theorem}{Theorem}
\newtheorem{lemma}{Lemma}
\newtheorem{proposition}{Proposition}
\newtheorem{corollary}{Corollary}
\newtheorem{definition}{Definition}
\newtheorem{remark}{Remark}

\usepackage[dvipsnames]{xcolor}
\definecolor{porange}{HTML}{E77500} %
\usepackage{amsmath}
\usepackage{amssymb}
\usepackage{amsfonts}
\usepackage{mathtools}
\usepackage{graphicx}
\usepackage{enumerate}
\usepackage{bm}
\usepackage[xindy,acronym,nowarn]{glossaries}
\usepackage{fontawesome5}

\usepackage{eso-pic}

\newcommand{\PreprintNotice}{%
  \AddToShipoutPictureFG*{%
    \AtPageLowerLeft{%
      \raisebox{28pt}[0pt][0pt]{%
        \makebox[\paperwidth][c]{%
          \footnotesize
          This work has been submitted to the IEEE for possible publication. Copyright may be transferred without notice, after which this version may no longer be accessible.%
        }%
      }%
    }%
  }%
}

\definecolor{porange}{HTML}{E77500} %
\definecolor{grey}{HTML}{919191} %

\DeclareDocumentEnvironment{example}{}{\noindent\textbf{Running example:}\itshape}{}

\usepackage{ifthen}
\newboolean{include-notes}
\newboolean{highlight-new}
\newboolean{include-remove}
\setboolean{include-notes}{true}
\setboolean{highlight-new}{true}
\setboolean{include-remove}{true}

\usepackage[dvipsnames]{xcolor}
\usepackage[normalem]{ulem}
\newcommand{\jaime}[1]{\ifthenelse{\boolean{include-notes}}{\textcolor{orange}{\textbf{Jaime:} #1}}{}}
\newcommand{\haimin}[1]{\ifthenelse{\boolean{include-notes}}{\textcolor{magenta}{\textbf{Haimin:} #1}}{}}
\newcommand{\justin}[1]{\ifthenelse{\boolean{include-notes}}{\textcolor{Cerulean}{\textbf{Justin:} #1}}{}}
\newcommand{\himani}[1]{\ifthenelse{\boolean{include-notes}}{\textcolor{Plum}{\textbf{Nishanth:} #1}}{}}
\newcommand{\david}[1]{\ifthenelse{\boolean{include-notes}}{\textcolor{teal}{\textbf{David:} #1}}{}}
\newcommand{\duy}[1]{\ifthenelse{\boolean{include-notes}}{\textcolor{blue}{\textbf{Duy:} #1}}{}}
\newcommand{\remove}[1]{\ifthenelse{\boolean{include-remove}}{\textcolor{red}{\sout{#1}}}{}}
\newcommand{\new}[1]{\ifthenelse{\boolean{highlight-new}}{\textcolor{blue}{#1}}{#1}}
\newcommand{\todo}[1]{\ifthenelse{\boolean{include-notes}}{\textcolor{blue}{\textbf{TODO:} #1}}{}}
\newcommand{\guy}[1]{\ifthenelse{\boolean{include-notes}}{\textcolor{Cerulean}{\textbf{Guy:} #1}}{}}
\newcommand{\flag}[1]{\ifthenelse{\boolean{include-notes}}{\textcolor{red}{#1}}{#1}}

\newcommand{\princeton}[1]{\ifthenelse{\boolean{include-notes}}{\textcolor{orange}{#1}}{}}

\newcommand{\p}[1]{\smallskip \noindent \textbf{{#1}.}}

\newcommand{\sRAS}{\textbf{\textcolor{porange}{sRAS}}}
\newcommand{\RA}{\textbf{\textcolor{black}{RA}}}
\newcommand{\AO}{\textbf{\textcolor{grey}{Avoid-only}}}

\newcommand{\ie}{\textit{i.e.}}

\newcommand{\failureset}{{\mathcal{F}}}

\newcommand{\task}{{\text{task}}}

\newcommand{\shield}{\text{\tiny{\faShield*}}}

\newcommand{\policyTask}{{\policy^\task}}

\newcommand{\safetyFilter}{{\phi}}

\newcommand{\reals}{\mathbb{R}}

\newcommand{\compl}{\mathsf{c}}

\DeclareMathOperator{\interior}{int}

\DeclareMathOperator*{\argmax}{arg\,max}
\DeclareMathOperator*{\argmin}{arg\,min}

\newcommand{\state}{{x}}

\newcommand{\ctrl}{{u}}

\newcommand{\dstb}{{d}}
\newcommand{\traj}{{\mathbf{x}}}%

\newcommand{\xset}{{\mathcal{X}}}  %

\newcommand{\cset}{{\mathcal{U}}}
\newcommand{\dset}{{\mathcal{D}}}

\newcommand{\dyn}{{f}}

\newcommand{\margin}{{g}}

\newcommand{\policy}{{\pi}}
\newcommand{\cpolicy}{\pi^\ctrl}
\newcommand{\dpolicy}{\pi^\dstb}

\newcommand{\stagecostsafe}{{\ell}}

\newcommand{\modeset}{\mathcal{M}}
\newcommand{\mode}{m}
\newcommand{\reachmode}{\mathrm R}
\newcommand{\staymode}{\mathrm S}

\newcommand{\targetset}{\mathcal{T}}
\newcommand{\staymargin}{g_{\mathrm S}}
\newcommand{\stayfailureset}{\mathcal{F}_{\mathrm S}}
\newcommand{\reachtargetset}{\mathcal{T}_{\mathrm R}}
\newcommand{\reachtargetmargin}{\ell_{\mathrm R}}
\newcommand{\reachmargin}{g_{\mathrm R}}
\newcommand{\reachfailureset}{\mathcal{F}_{\mathrm R}}
\newcommand{\stayval}{V_{\mathrm S}}
\newcommand{\reachval}{V_{\mathrm R}}
\newcommand{\stayq}{\mathcal{Q}_{\mathrm S}}
\newcommand{\reachq}{\mathcal{Q}_{\mathrm R}}
\newcommand{\stayset}{\Omega_{\mathrm S}^{*}}

\newcommand{\srasset}{\Omega_{\mathrm{sRAS}}^{*}}
\newcommand{\raset}{\Omega_{\mathrm{RA}}^{*}}
\newcommand{\stayfallback}{\policy^{\shield}_{\mathrm S}}
\newcommand{\reachfallback}{\policy^{\shield}_{\mathrm R}}
\newcommand{\classK}{\mathcal{K}}

\newglossarystyle{mylist}{%
  \setglossarystyle{list}%
}

\glsdisablehyper
\makeglossaries

\newglossaryentry{GDA}
{
  name={GDA},
  description={gradient descent-ascent},
  first={gradient descent-ascent (\glsentrytext{GDA})}
}

\newglossaryentry{RL}
{
  name={RL},
  description={reinforcement learning},
  first={reinforcement learning (\glsentrytext{RL})}
}

\newglossaryentry{HJ}
{
  name={HJ},
  description={Hamilton--Jacobi},
  first={Hamilton--Jacobi (\glsentrytext{HJ})}
}

\newglossaryentry{HJI}
{
  name={HJI},
  description={Hamilton--Jacobi--Isaacs},
  first={Hamilton--Jacobi--Isaacs (\glsentrytext{HJI})}
}

\newglossaryentry{DCBF}
{
  name={DCBF},
  plural={DCBFs},
  description={discrete-time control barrier function},
  first={discrete-time control barrier function (\glsentrytext{DCBF})},
  firstplural={discrete-time control barrier functions (\glsentryplural{DCBF})}
}

\newglossaryentry{CBF}
{
  name={CBF},
  plural={CBFs},
  description={control barrier function},
  first={control barrier function (\glsentrytext{CBF})},
  firstplural={control barrier functions (\glsentryplural{CBF})}
}

\newglossaryentry{Q-CBF}
{
  name={Q-CBF},
  plural={Q-CBFs},
  description={state--action control barrier function},
  first={state--action control barrier function (\glsentrytext{Q-CBF})},
  firstplural={state--action control barrier functions (\glsentryplural{Q-CBF})}
}

\newglossaryentry{HCBF}
{
  name={HCBF},
  plural={HCBFs},
  description={hybrid control barrier function},
  first={hybrid control barrier function (\glsentrytext{HCBF})},
  firstplural={hybrid control barrier functions (\glsentryplural{HCBF})}
}

\newglossaryentry{VB-CBF}
{
  name={VB-CBF},
  plural={VB-CBFs},
  description={viscosity-based control barrier function},
  first={viscosity-based control barrier function (\glsentrytext{VB-CBF})},
  firstplural={viscosity-based control barrier functions (\glsentryplural{VB-CBF})}
}

\newglossaryentry{DDPG}
{
  name={DDPG},
  description={deep deterministic policy gradient},
  first={deep deterministic policy gradient (\glsentrytext{DDPG})}
}

\newglossaryentry{ODD}
{
  name={ODD},
  description={Operational Design Domain},
  first={operational design domain (\glsentrytext{ODD})}
}

\newglossaryentry{LRSF}
{
  name={LRSF},
  description={Least-Restrictive Safety Filter},
  first={least-restrictive safety filter (\glsentrytext{LRSF})}
}

\newglossaryentry{HCSF}
{
  name={HCSF},
  description={Human-Centered Safety Filter},
  first={human-centered safety filter (\glsentrytext{HCSF})}
}

\newglossaryentry{NLP}
{
  name={NLP},
  description={nonlinear programming problem},
  first={nonlinear programming problem (\glsentrytext{NLP})},
}

\newglossaryentry{ILQR}
{
  name={ILQR},
  description={iterative linear quadratic regulator},
  first={iterative linear quadratic regulator (\glsentrytext{ILQR})},
}

\newglossaryentry{MPC}
{
  name={MPC},
  description={model predictive control},
  first={model predictive control (\glsentrytext{MPC})},
}

\newglossaryentry{AC}
{
  name={AC},
  description={Assetto Corsa},
  first={Assetto Corsa (\glsentrytext{AC})},
}

\newglossaryentry{SAC}
{
  name={SAC},
  description={Soft Actor--Critic},
  first={soft actor--critic (\glsentrytext{SAC})},
}

\newglossaryentry{ANOVA}
{
    name={ANOVA},
    description={Analysis of Variance},
    first={analysis of variance (\glsentrytext{ANOVA})}
}

\newglossaryentry{SME}
{
    name={SME},
    description={Simple Main Effects},
    first={simple main effects (\glsentrytext{SME})},
}

\newglossaryentry{HSD}
{
    name={HSD},
    description={Honestly Significant Difference},
    first={honestly significant difference (\glsentrytext{HSD})},
}

\newglossaryentry{ECDF}
{
    name={ECDF},
    description={Empirical Cumulative Distribution Function},
    first={empirical cumulative distribution function (\glsentrytext{ECDF})},
}

\newglossaryentry{OCP}
{
    name={OCP},
    description={Optimal Control Problem},
    first={optimal control problem (\glsentrytext{OCP})}
}

\newglossaryentry{AI}
{
    name={AI},
    description={Artificial Intelligence},
    first={artificial intelligence (\glsentrytext{AI})}
}

\newglossaryentry{HRI}
{
    name={HRI},
    description={Human--Robot Interaction},
    first={human--robot interaction (\glsentrytext{HRI})}
}

\newglossaryentry{MDP}
{
    name={MDP},
    description={Markov Decision Process},
    first={Markov decision process (\glsentrytext{MDP})}
}

\newglossaryentry{SC-MDP}
{
    name={SC-MDP},
    description={Safety-Critical Markov Decision Process},
    first={safety-critical Markov decision process (\glsentrytext{SC-MDP})}
}

\newacronym[longplural={constrained Markov decision processes}]{CMDP}{CMDP}{constrained Markov decision process}

\newglossaryentry{DP}
{
    name={DP},
    description={Dynamic Programming},
    first={dynamic programming (\glsentrytext{DP})}
}

\newglossaryentry{CPO}
{
    name={CPO},
    description={Constrained Policy Optimization},
    first={Constrained Policy Optimization (\glsentrytext{CPO})}
}

\newglossaryentry{RCPO}
{
    name={RCPO},
    description={Reward Constrained Policy Optimization},
    first={Reward Constrained Policy Optimization (\glsentrytext{RCPO})}
}

\newglossaryentry{RAS}
{
    name={RAS},
    description={Reach--Avoid--Stay},
    first={reach--avoid--stay (\glsentrytext{RAS})}
}

\newglossaryentry{sRAS}
{
    name={sRAS},
    description={strict Reach--Avoid--Stay},
    first={strict reach--avoid--stay (\glsentrytext{sRAS})}
}

\newglossaryentry{RA}
{
    name={RA},
    description={Reach--Avoid},
    first={reach--avoid (\glsentrytext{RA})}
}

\newacronym{MPSF}{MPSF}{model predictive safety filter}

\newacronym{ISAACS}{ISAACS}{Iterative Soft Adversarial Actor--Critic for Safety}

\newacronym{QP}{QP}{quadratic program}

\newacronym[longplural={Gaussian processes}]{GP}{GP}{Gaussian process}

\newacronym{CVaR}{CVaR}{conditional value-at-risk}

\newglossaryentry{STL}
{
    name={STL},
    description={signal temporal logic},
    first={signal temporal logic (\glsentrytext{STL})}
}

\newglossaryentry{CB-VF}
{
    name={CB-VF},
    plural={CB-VFs},
    description={control barrier--value function},
    first={control barrier--value function (\glsentrytext{CB-VF})},
    firstplural={control barrier--value functions (\glsentryplural{CB-VF})}
}

\usepackage{capt-of} %

\title{\LARGE \bf
Winning a Won Game: Strict Reach--Avoid--Stay Control Barrier Functions for High-Dimensional Black-Box Systems
}

\author{Donggeon David Oh$^{1}$, Duy P.~Nguyen$^{1}$, Gongkai Yuan$^{2}$,\\
Qingchen Li$^{2}$, Jaime Fern\'andez Fisac$^{1,\dagger}$, and Haimin Hu$^{2,\dagger}$%
\thanks{$^{1}$D. D. Oh, D. P. Nguyen, and J. F. Fisac are with the Department of Electrical and Computer Engineering, Princeton University, Princeton, NJ 08544, USA.
{\tt\small \{do9948,duyn,jfisac\}@princeton.edu}}%
\thanks{$^{2}$G. Yuan, Q. Li, and H. Hu are with the Department of Computer Science, Johns Hopkins University, Baltimore, MD 21218, USA.
{\tt\small \{gyuan3,qli137\}@jh.edu, haimin@cs.jhu.edu}}%
\thanks{$^{\dagger}$Equal advising.}%
}

\begin{document}

\makeatletter
\let\@oldmaketitle\@maketitle
\renewcommand{\@maketitle}{%
  \@oldmaketitle%
  \begin{minipage}{\textwidth}
  \setcounter{figure}{0}
\centering
\includegraphics[width=\textwidth]{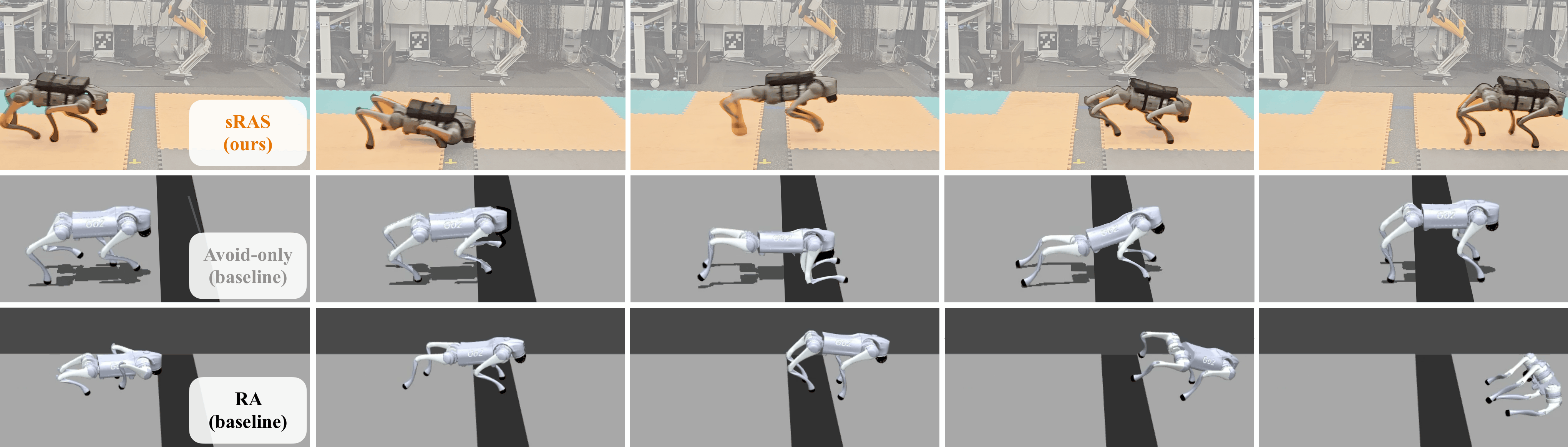}%

\captionof{figure}{A Unitree Go2 quadruped must jump across a $25\,\mathrm{cm}$ gap, land safely on the far side, and remain safe afterward.
This is a strict reach--avoid--stay (\sRAS) task, where a robot must safely reach a target and remain safely within it indefinitely after first target entry.
All methods use the same task policy, which is a pure-pursuit policy that walks in the right direction.
Sequences progress from left to right.
\textbf{Top (hardware):} With our \sRAS~$Q$-control barrier function (CBF) safety filter, the robot builds momentum, crouches for takeoff, jumps across the gap, lands safely, and remains safe. The same successful behavior is observed in simulation.
\textbf{Middle (simulation):} The \AO~safety filter does not encode reaching the far side; the robot fails to complete the jump, leaving its hind legs on the starting side.
\textbf{Bottom (simulation):} The reach--avoid (\RA) safety filter encodes reaching the target safely but not remaining safe afterward; the robot jumps across the gap but fails to land safely.
The successful hardware execution illustrates the intended behavior of our \sRAS~$Q$-CBF filter: modifying task actions to keep the robot in \sRAS~winnable states and maintain safety within the target after first target entry.}
\label{fig:go2_gap_hardware}
  \end{minipage}\par\vspace{-0.3cm}
}
\makeatother

\maketitle
\thispagestyle{empty}
\pagestyle{empty}

\PreprintNotice

\begin{abstract}
Robots must complete their tasks and maintain the achieved outcomes while avoiding safety failures at all times.
Strict reach--avoid--stay (sRAS) formalizes this requirement: safely reaching a target and remaining there indefinitely after first entry.
We propose an sRAS $Q$-control barrier function (CBF) safety filter for high-dimensional black-box systems under bounded uncertainty.
Our construction combines a stay value encoding safe permanent residence in a target subset with a reach--avoid value encoding safe reachability of this subset while avoiding target states from which safe permanent residence cannot be guaranteed.
We prove that these values jointly yield a valid robust discrete-time CBF and lift them to state--action $Q$-functions for runtime intervention.
For exact values and under a measure-zero condition, our filter preserves sRAS feasibility from almost every winnable initial state and keeps the system safely within the target after first entry, against all admissible uncertainty realizations.
We adopt reachability-based adversarial reinforcement learning for scalable value approximation using only black-box interactions.
Notably, neither synthesis nor deployment of our filter requires known dynamics, affine structure, value derivatives, or hand-designed barriers.
We validate our framework in quadruped gap jumping in simulation and hardware, where the robot crosses the gap, lands safely, and remains safe afterward.
Simulated F1TENTH races further demonstrate safe overtaking and lead retention.
\end{abstract}

\vspace{0.1cm}
{\small\itshape
\noindent The hardest game to win is a won game.\par
\vspace{0.1cm}
\raggedleft\upshape\textsc{Emanuel Lasker, World Chess Champion}
\par}

\section{Introduction}
\label{sec:introduction}
As robots take on increasingly complex tasks in the real world, they must operate safely under uncertainty.
\emph{Safety filters} address this need by monitoring operation at runtime and intervening when necessary by modifying proposed controls to prevent catastrophic failures~\cite{hsu2024safety}.
\emph{\Gls{CBF}}-based filters are a popular choice because they can provide smooth safety interventions~\cite{hsu2024safety,oh2025safety,oh2026synthesis}.

However, safety interventions should also \emph{preserve the ability to complete the task and maintain the achieved outcome}.
For example, a quadruped tasked with crossing a gap should build momentum, jump, land safely, and remain safe after touchdown, rather than stay on the starting side or jump only to fall after touchdown (Figure~\ref{fig:go2_gap_hardware}).
In car racing, the car should overtake safely and retain the lead, rather than remain behind an opponent or get overtaken again (Figure~\ref{fig:f1_overtaking_circuit}).
We formalize this requirement as \emph{\gls{sRAS}}, which requires safely reaching a predefined target and remaining there indefinitely after the first entry.

We consider \emph{high-dimensional systems} under bounded uncertainty with only \emph{black-box access to their dynamics}.
\emph{Our goal is to devise a \gls{CBF}-based safety filter for such systems that preserves \gls{sRAS} feasibility, \ie, keeps the system in states from which the specification remains achievable.}

Existing work extends \glspl{CBF} to encode similar temporal specifications through time-varying \glspl{CBF}~\cite{lindemann2018control,das2025control,das2025controlb} and control Lyapunov--barrier methods~\cite{meng2021control,zhou2026patching}, but these approaches rely on known control-affine dynamics, preventing their direct application to black-box systems.
On the other hand, reachability-based \gls{RL} enables scalable value learning for complex temporal specifications through black-box interaction~\cite{hsu2021safety,nguyen2025gameplay,chenevert2024solving,sharpless2026dual,sharpless2026bellman,so2026value}.
In parallel, exact reachability values are valid \glspl{CBF} in a generalized sense, with runtime filtering requiring known dynamics and value derivatives~\cite{choi2021robust,hirsch2026viscosity,begzadic2025back}, while learned values have been used as \glspl{DCBF} for avoid-only tasks on black-box systems~\cite{oh2025safety,oh2026synthesis}.
We review these approaches and compare them with our work in \autoref{sec:related_work}.

Building upon these works, we propose the \emph{\gls{sRAS} $Q$-\gls{CBF} safety filter}.
Our main contributions are as follows:

\noindent \textbf{\glsentrytext{sRAS} \glsentrytext{Q-CBF} theory.}
We propose the \emph{\gls{sRAS} $Q$-\gls{CBF} safety filter} for \emph{high-dimensional black-box systems under bounded uncertainty}.
We prove that the \emph{stay value}, encoding safe permanent residence in a target subset, and the \emph{reach--avoid value}, encoding reachability of this subset while avoiding failure states and target states from which safe permanent residence cannot be guaranteed, jointly yield a valid robust \gls{DCBF}.
For exact values and under a measure-zero condition, our filter 1) \emph{is recursively feasible}, 2) \emph{preserves \gls{sRAS} feasibility from almost every winnable initial state}, and 3) \emph{keeps the system safely within the target indefinitely after first entry} against all admissible uncertainty realizations.

\noindent \textbf{Scalable black-box synthesis and deployment.}
We adopt reachability-based adversarial \gls{RL} for scalable synthesis of stay and reach--avoid values and their state--action lifts using only black-box interactions.
The lifted values, together with the learned worst-case response disturbance policies, enable direct evaluation of $Q$-\gls{CBF} constraints for candidate controls.
Thus, \emph{neither synthesis nor deployment requires known dynamics, control- or disturbance-affine structure, value derivatives, or hand-designed barrier candidates}.

\noindent \textbf{Robotic evaluation.}
We evaluate our \gls{sRAS} $Q$-\gls{CBF} safety filter on two robotic tasks: 1) quadruped gap jumping in simulation and hardware, where the robot crosses the gap, lands safely, and remains safe afterward, and 2) simulated F1TENTH endurance racing, where the ego vehicle overtakes opponents safely and retains the lead.

\section{Related Work}
\label{sec:related_work}

\p{\glspl{CBF} for temporal specifications}
Prior work extends \glspl{CBF} beyond failure avoidance to encode temporal specifications.
Time-varying \glspl{CBF} are often adopted to capture the semantics of \gls{STL} tasks~\cite{lindemann2018control}.
Das et al.~\cite{das2025control,das2025controlb} synthesize spatiotemporal tubes for \gls{STL} tasks, including prescribed-time \gls{RAS}, and derive corresponding time-varying \glspl{CBF}.
Control Lyapunov--barrier methods that induce safe stabilization enable the synthesis of \gls{RAS} controllers~\cite{meng2021control} and conservative approximations of local winning sets for decomposed temporal logic tasks~\cite{zhou2026patching}.
These approaches require known control-affine dynamics for controller synthesis and deployment.
By contrast, our \gls{sRAS} $Q$-\gls{CBF} framework enables end-to-end synthesis and deployment on high-dimensional black-box systems while preserving the ability to satisfy the \gls{sRAS} specification from almost every state in its maximal robust winning set under a measure-zero assumption.

\p{Reachability analysis for temporal specifications}
Decomposing complex temporal logic specifications into reachability subproblems has been adopted for devising least-restrictive feasible controller sets for \gls{STL}~\cite{chen2018signal} and constructing temporal logic trees for online control synthesis~\cite{gao2021temporal}.
Recent advances in reachability-based \gls{RL} enable scalable solutions to these subproblems.
Time-discounted Bellman equations enable neural approximation of reachability values via \gls{RL}~\cite{fisac2019bridging,hsu2021safety}, and extensions to uncertain dynamics yield a predictive safety filter based on rollouts under disturbance~\cite{nguyen2025gameplay}.
These methods support learning optimal policies~\cite{sharpless2026dual,sharpless2026bellman} and constructing least-restrictive safety filters~\cite{so2026value} for complex temporal logic tasks decomposed into graphs of reachability subproblems.
\gls{RAS} has also been decomposed into avoid and reach--avoid subproblems, with corresponding policies learned via \gls{RL}~\cite{chenevert2024solving}.
Our work similarly decomposes the \gls{sRAS} specification into two reachability subproblems and uses reachability-based \gls{RL} to learn their values through black-box interactions.
However, we differ from this line of work in that we devise a \gls{CBF} safety filter that preserves \gls{sRAS} feasibility under bounded uncertainty.

\p{Reachability-derived \glspl{CBF}}
In continuous-time systems, the infinite-horizon avoid value is known to be a valid \gls{CBF} when continuously differentiable~\cite{hsu2024safety}.
Robust \glspl{CB-VF}, constructed as viscosity solutions of a modified Hamilton--Jacobi--Isaacs variational inequality, enable \gls{CBF}-like filtering~\cite{choi2021robust}, while viscosity \glspl{CBF} that satisfy the \gls{CBF} condition in the viscosity sense are equivalent to time-invariant \glspl{CB-VF} if nonnegative~\cite{hirsch2026viscosity}.
A finite-horizon reach--avoid value is also shown to be a valid time-varying viscosity-based \gls{CBF}~\cite{begzadic2025back}.
Although reachability values can recover maximal winning sets, runtime evaluation of the corresponding \gls{CBF} constraints requires known dynamics and derivatives of the value or a viscosity test function.
In discrete time, avoid values yield valid \glspl{DCBF}, with the lifted value function enabling scalable black-box synthesis and deployment~\cite{oh2025safety,oh2026synthesis}.
However, these avoid-only values do not by themselves encode task-completion feasibility for temporal specifications.
Our \gls{sRAS} $Q$-\gls{CBF} enables black-box synthesis and deployment without known dynamics, control- or disturbance-affine structure, or value derivatives while preserving the ability to satisfy the \gls{sRAS} specification.

\section{Preliminaries and Problem Formulation}
\label{sec:problem}

\subsection{Strict Reach--Avoid--Stay under Bounded Uncertainty}

We consider a nonlinear discrete-time system
\begin{equation}
\label{eq:dynamics}
\state_{t+1}=\dyn(\state_t,\ctrl_t,\dstb_t),
\end{equation}
where $\state_t\in\xset\subseteq\mathbb R^n$ is the state, $\ctrl_t\in\cset$ is the control input, and $\dstb_t\in\dset$ is an unknown but bounded disturbance input representing predictive uncertainty.
We adopt a black-box viewpoint in which the transition mechanism in~\eqref{eq:dynamics} can be queried for the next state given the current state, control, and disturbance, but need not be available in closed form or possess control- or disturbance-affine structure.

We consider \gls{sRAS} tasks, with a \emph{failure set} $\failureset\subset\xset$ containing unacceptable states that must be avoided at all times, and a \emph{target set} $\targetset\subset\xset$ containing states that the system should reach and remain within after its first entry.
Throughout, robust satisfaction requires a single control policy $\cpolicy:\xset\to\cset$ to satisfy the specification against every disturbance policy $\dpolicy:\xset\to\dset$.

\begin{definition}[Maximal Robust \glsentrytext{sRAS} Set]
\label{def:sras_set}
The maximal robust \gls{sRAS} set is
\begin{equation}
\label{eq:max_sras_set}
\srasset
:=
\left\{
\state\in\xset
\middle|
\begin{array}{l}
\exists\cpolicy,\ \forall\dpolicy,\ \exists\tau\ge0\ \mathrm{s.t.}\\[-1mm]
\traj_{\state}^{\cpolicy,\dpolicy}(t)\notin\failureset,
\quad \forall t\ge0,\\[-1mm]
\traj_{\state}^{\cpolicy,\dpolicy}(t)\notin\targetset,
\quad \forall t<\tau,\\[-1mm]
\traj_{\state}^{\cpolicy,\dpolicy}(t)\in\targetset,
\quad \forall t\ge\tau
\end{array}
\right\}.
\end{equation}
Throughout, $\traj_{\state}^{\cpolicy,\dpolicy}:\mathbb{Z}_{\ge0}\to\xset$ denotes the state trajectory initialized at $\traj_{\state}^{\cpolicy,\dpolicy}(0)=\state$ and evolved according to $\cpolicy:\xset\to\cset$, $\dpolicy:\xset\to\dset$, and the dynamics~\eqref{eq:dynamics}.
\end{definition}

Conventional \gls{RAS} omits the condition $\traj_{\state}^{\cpolicy,\dpolicy}(t)\notin\targetset$, $\forall t<\tau$, in~\eqref{eq:max_sras_set}, allowing repeated target entry and exit before eventually remaining there permanently.
This is undesirable for robotic tasks that require irreversible task completion---in car racing, the ego vehicle should not fall behind again after an overtake---making \gls{sRAS} a more suitable specification.
Any \gls{sRAS}-feasible state $\state\in\xset$ is also \gls{RAS}-feasible.

\subsection{Robust Discrete-Time Control Barrier Function}
\label{subsec:rdcbf}

A robust \gls{DCBF} renders its $0$-superlevel set robustly controlled invariant by requiring a control input whose worst-case next-step barrier value is lower-bounded by a class-$\classK$ mapping of its current value.
We adopt the definition in~\cite{oh2026synthesis}.

\begin{definition}[Robust Discrete-Time \glsentrytext{CBF}]
\label{def:rdcbf}
A function $h:\xset\to\reals$ is a robust \gls{DCBF} for system~\eqref{eq:dynamics} if $\Omega_h:=\{\state\in\xset\mid h(\state)\ge0\}$ satisfies $\Omega_h\cap\failureset=\varnothing$, and there exists a class-$\classK$ function $\beta$, defined on a domain containing $h(\Omega_h)$, such that
\begin{equation}
\label{eq:rdcbf_condition}
\sup_{\ctrl\in\cset}
\inf_{\dstb\in\dset}
h\bigl(\dyn(\state,\ctrl,\dstb)\bigr)
\ge
\beta\bigl(h(\state)\bigr),
\quad
\forall\state\in\Omega_h.
\end{equation}
\end{definition}

$\Omega_h$ may in general be a strict subset of the maximal robust safe set from which failure can be avoided indefinitely against every admissible disturbance policy~\cite{oh2026synthesis}.

\section{Decomposing the \glsentrytext{sRAS} Game}
\label{sec:decomposition}
We introduce a two-stage reachability-based value synthesis pipeline for robust \gls{sRAS}.
Stage I computes the maximal robust terminal safe set, the largest robust controlled invariant subset of $\targetset\setminus\failureset$.
Stage II computes the maximal robust reach--avoid set from which this terminal safe set can be reached while avoiding $\failureset$ and interior target states outside the terminal safe set.
For each stage, we lift the corresponding value into the state--control--disturbance space through $Q$-functions~\cite{watkins1992q}, which enables \gls{CBF} constraint evaluation for black-box systems.
Finally, we prove almost-everywhere recovery of the maximal robust \gls{sRAS} set via our two-stage synthesis pipeline under an assumption.
Throughout, we assume that all displayed maxima and minima are attained.

We choose continuous \emph{safety and target margins} $\margin,\stagecostsafe:\xset\to\reals$ satisfying
\begin{equation}
\label{eq:task_margins}
\failureset
=
\{\state\in\xset\mid\margin(\state)<0\},
\quad
\targetset
=
\{\state\in\xset\mid\stagecostsafe(\state)\ge0\},
\end{equation}
with $\interior(\targetset)=\{\stagecostsafe>0\}$ and $\partial\targetset=\{\stagecostsafe=0\}$.

\subsection{Stage I: Stay Value}
\label{subsec:stay_value}
We define the \emph{stay-mode failure set} as
\begin{equation}
\label{eq:stay_failure_set}
\stayfailureset
:=
\failureset\cup\targetset^{\compl}.
\end{equation}

\begin{definition}[Maximal Robust Terminal Safe Set]
\label{def:terminal_set}
The maximal robust terminal safe set is
\begin{equation}
\label{eq:max_terminal_set_semantic}
\stayset
:=
\left\{
\state\in\xset
\middle|
\exists\cpolicy,\ \forall\dpolicy,\ \forall t\ge0,\ 
\traj_{\state}^{\cpolicy,\dpolicy}(t)\notin\stayfailureset
\right\}.
\end{equation}
\end{definition}

To transform this binary safety outcome into a ``game of degree'', we choose the \emph{stay-mode safety margin}
\begin{equation}
\label{eq:stay_failure_margin}
\staymargin(\state)
:=
\min\left\{\margin(\state),\stagecostsafe(\state)\right\},
\end{equation}
which satisfies $\stayfailureset=\{\staymargin<0\}$.

The \emph{stay value} $\stayval$ is then defined as the maximal $\staymargin$ that can be maintained against the worst-case admissible disturbance policy over the infinite horizon,
\begin{equation}
\label{eq:stay_value_trajectory}
\stayval(\state)
:=
\max_{\cpolicy}\min_{\dpolicy}\inf_{t\ge0}
\staymargin\!\left(\traj_{\state}^{\cpolicy,\dpolicy}(t)\right),
\end{equation}
and $\stayval$ satisfies the avoid-only Isaacs equation
\begin{equation}
\label{eq:stay_bellman}
\stayval(\state)
=
\min\left\{
\staymargin(\state),
\max_{\ctrl\in\cset}\min_{\dstb\in\dset}
\stayval\bigl(\dyn(\state,\ctrl,\dstb)\bigr)
\right\}.
\end{equation}
$\stayset$ in~\eqref{eq:max_terminal_set_semantic} is recovered as the $0$-superlevel set of $\stayval$:
\begin{equation}
\label{eq:stay_set_value}
\stayset
=
\left\{\state\in\xset\mid\stayval(\state)\ge0\right\}.
\end{equation}

To directly evaluate candidate control inputs under disturbance realizations, we define the \emph{lifted stay value}
\begin{equation}
\label{eq:stay_q}
\stayq(\state,\ctrl,\dstb)
:=
\min\left\{
\staymargin(\state),
\stayval\bigl(\dyn(\state,\ctrl,\dstb)\bigr)
\right\}.
\end{equation}
For fixed $\state$, the map $z\mapsto\min\{\staymargin(\state),z\}$ is monotone nondecreasing, so it preserves the attained inner minimum and outer maximum, yielding
\begin{equation}
\label{eq:stay_lift}
\stayval(\state)
=
\max_{\ctrl\in\cset}\min_{\dstb\in\dset}
\stayq(\state,\ctrl,\dstb).
\end{equation}
The corresponding stay fallback policy is
\begin{equation}
\label{eq:stay_fallback}
\stayfallback(\state)
\in
\argmax_{\ctrl\in\cset}
\min_{\dstb\in\dset}
\stayq(\state,\ctrl,\dstb).
\end{equation}

\subsection{Stage II: Reach--Avoid Value}
\label{subsec:reach_avoid_value}
We define the \emph{reach-mode target and failure sets} as
\begin{equation}
\label{eq:reach_sets}
\reachtargetset
:=
\stayset,
\qquad
\reachfailureset
:=
\failureset\cup\bigl(\interior(\targetset)\setminus\reachtargetset\bigr),
\end{equation}
so that the reach mode terminates only upon entry into $\stayset$.
$\reachfailureset$ does not include safe nonviable states on $\partial\targetset$, since $\failureset\cup\bigl(\targetset\setminus\reachtargetset\bigr)$ cannot in general be represented exactly as a strict sublevel set of a continuous safety margin.
We quantify the resulting discrepancy in \autoref{subsec:set_characterization}.

\begin{definition}[Maximal Robust Reach--Avoid Set]
\label{def:ra_set}
The maximal robust reach--avoid set is
\begin{equation}
\label{eq:max_ra_set_semantic}
\raset
:=
\left\{
\state\in\xset
\middle|
\begin{array}{l}
\exists\cpolicy,\ \forall\dpolicy,\ \exists\tau\ge0\ \mathrm{s.t.}\\[-1mm]
\traj_{\state}^{\cpolicy,\dpolicy}(\tau)\in\reachtargetset,\\[-1mm]
\traj_{\state}^{\cpolicy,\dpolicy}(t)\notin\reachfailureset,\,\,\,\forall 0\le t\le\tau
\end{array}
\right\}.
\end{equation}
\end{definition}

Since $\reachtargetset=\stayset\subseteq\targetset\setminus\failureset$ and $\stayset\cap\reachfailureset=\varnothing$, every $\state\in\stayset$ satisfies \autoref{def:ra_set} with $\tau=0$.
Therefore,
\begin{equation}
\label{eq:stay_ra_inclusion}
\stayset\subseteq\raset.
\end{equation}

We choose the \emph{reach-mode target margin} $\reachtargetmargin$ and \emph{reach-mode safety margin} $\reachmargin$ as
\begin{subequations}
\label{eq:reach_margins}
\begin{align}
\reachtargetmargin(\state)
&:=
\stayval(\state),
\label{eq:reach_target_margin}\\
\reachmargin(\state)
&:=
\min\left\{
\margin(\state),
\max\left\{-\stagecostsafe(\state),\reachtargetmargin(\state)\right\}
\right\}.
\label{eq:reach_failure_margin}
\end{align}
\end{subequations}
The target margin recovers the reach-mode target set,
\begin{equation}
\label{eq:reach_target_margin_set}
\reachtargetset
=
\left\{\state\in\xset\mid\reachtargetmargin(\state)\ge0\right\},
\end{equation}
while the safety margin recovers the reach-mode failure set,
\begin{equation}
\label{eq:reach_failure_margin_set}
\begin{aligned}
\{\reachmargin<0\}
&=
\{\margin<0\}\cup
\bigl(\{\stagecostsafe>0\}\cap\{\reachtargetmargin<0\}\bigr)\\
&=
\failureset\cup\bigl(\interior(\targetset)\setminus\reachtargetset\bigr)
=
\reachfailureset.
\end{aligned}
\end{equation}

The \emph{reach--avoid value} $\reachval$ is defined by jointly considering the $\reachtargetmargin$ at arrival and $\reachmargin$ maintained up to arrival under the worst-case admissible disturbance policy:
\begin{equation}
\label{eq:reach_value_trajectory}
\begin{aligned}
\reachval(\state)
:=
\max_{\cpolicy}\min_{\dpolicy}\max_{\tau\ge0}
\min&\Bigl\{\reachtargetmargin\!\left(\traj_{\state}^{\cpolicy,\dpolicy}(\tau)\right),\\
&\min_{0\le t\le\tau}
\reachmargin\!\left(\traj_{\state}^{\cpolicy,\dpolicy}(t)\right)
\Bigr\},
\end{aligned}
\end{equation}
and $\reachval$ satisfies the reach--avoid Isaacs equation
{\small
\begin{equation}
\label{eq:reach_bellman}
\reachval(\state)
=
\min\!\left\{
\reachmargin(\state),
\max\!\left\{
\reachtargetmargin(\state),
\max_{\ctrl\in\cset}\min_{\dstb\in\dset}
\reachval\bigl(\dyn(\state,\ctrl,\dstb)\bigr)
\right\}
\right\}.
\end{equation}
}
$\raset$ in~\eqref{eq:max_ra_set_semantic} is recovered as the $0$-superlevel set of $\reachval$:
\begin{equation}
\label{eq:ra_set_value}
\raset
=
\left\{\state\in\xset\mid\reachval(\state)\ge0\right\}.
\end{equation}

To directly evaluate candidate control inputs under disturbance realizations, we define the \emph{lifted reach--avoid value}
\begin{equation}
\label{eq:reach_q}
\reachq(\state,\ctrl,\dstb)
:=
\min\!\left\{
\reachmargin(\state),
\max\!\left\{
\reachtargetmargin(\state),
\reachval\bigl(\dyn(\state,\ctrl,\dstb)\bigr)
\right\}
\right\}.
\end{equation}
For fixed $\state$, the map $z\mapsto\min\{\reachmargin(\state),\max\{\reachtargetmargin(\state),z\}\}$ is monotone nondecreasing, and therefore
\begin{equation}
\label{eq:reach_lift}
\reachval(\state)
=
\max_{\ctrl\in\cset}\min_{\dstb\in\dset}
\reachq(\state,\ctrl,\dstb).
\end{equation}
The corresponding reach--avoid fallback policy is
\begin{equation}
\label{eq:reach_fallback}
\reachfallback(\state)
\in
\argmax_{\ctrl\in\cset}
\min_{\dstb\in\dset}
\reachq(\state,\ctrl,\dstb).
\end{equation}

\subsection{Maximal Robust \glsentrytext{sRAS} Set Characterization}
\label{subsec:set_characterization}
To characterize the difference between $\raset$ and $\srasset$ caused by the target-boundary discrepancy, we isolate the relevant portion of the $0$-level set of $\reachval$ as
\begin{equation}
\label{eq:zero_level_set}
\mathcal Z_{\mathrm R}
:=
\{\state\in\raset\setminus\stayset\mid\reachval(\state)=0\}.
\end{equation}

\begin{proposition}[\glsentrytext{sRAS} Set Characterization]
\label{prop:sras_set_characterization}
\begin{equation}
\label{eq:sequential_set_inclusion}
\raset\setminus\mathcal Z_{\mathrm R}
\subseteq
\srasset
\subseteq
\raset
.
\end{equation}
\end{proposition}

\begin{proof}
We first show $\raset\setminus\mathcal Z_{\mathrm R}\subseteq\srasset$.
Fix $\state\in\raset\setminus\mathcal Z_{\mathrm R}$.
If $\state\in\stayset$, then \eqref{eq:stay_set_value}--\eqref{eq:stay_fallback} imply that $\stayfallback$ robustly renders $\stayset\subseteq\targetset\setminus\failureset$ invariant, so \autoref{def:sras_set} holds with $\tau=0$.

Otherwise if $\state\in\raset\setminus(\stayset\cup\mathcal Z_{\mathrm R})$, $\reachval(\state)>0$ by \eqref{eq:ra_set_value} and \eqref{eq:zero_level_set}, so \eqref{eq:reach_value_trajectory} gives a policy $\pi_{\mathrm R}^{u}$ such that, for every $\dpolicy$, the trajectory $\state_t:=\traj_{\state}^{\pi_{\mathrm R}^{u},\dpolicy}(t)$ admits $\bar\tau\ge0$ with $\reachtargetmargin(\state_{\bar\tau})>0$ and $\reachmargin(\state_t)>0,\ \forall t\le\bar\tau$.
Let $\tau\le\bar\tau$ be the first entry time into $\stayset$, which exists since $\state_{\bar\tau}\in\stayset$.
For all $t<\tau$, $\reachtargetmargin(\state_t)<0$, so \eqref{eq:reach_failure_margin} gives $\margin(\state_t)>0$ and $\stagecostsafe(\state_t)<0$, hence $\state_t\notin\failureset\cup\targetset$.
The policy applying $\pi_{\mathrm R}^{u}$ outside $\stayset$ and $\stayfallback$ inside $\stayset$ follows $\state_t$ through $\tau$ and remains in $\stayset$ thereafter, satisfying \autoref{def:sras_set} with $\tau$, thus proving $\state\in\srasset$.

To show $\srasset\subseteq\raset$, fix $\state\in\srasset$ and a policy $\cpolicy$ satisfying \autoref{def:sras_set}.
For any $\dpolicy$, let $\state_t:=\traj_{\state}^{\cpolicy,\dpolicy}(t)$ and let $\tau$ be its first entry time into $\targetset$.
To prove $\state_\tau\in\stayset$, take any $\tilde{\pi}^{d}$ and define $\hat{\pi}^{d}$ to equal $\dpolicy$ on $\targetset^{\compl}$ and $\tilde{\pi}^{d}$ on $\targetset$.
Since $\state_t\notin\targetset,\ \forall t<\tau$, $\hat{\state}_t:=\traj_{\state}^{\cpolicy,\hat{\pi}^{d}}(t)=\state_t,\ \forall t\le\tau$ and $\hat{\state}_t\in\targetset\setminus\failureset,\ \forall t\ge\tau$ by \autoref{def:sras_set}.
Now let $\tilde{\state}_{\tilde t}:=\traj_{\state_\tau}^{\cpolicy,\tilde{\pi}^{d}}(\tilde t)$.
Since $\hat{\state}_t\in\targetset,\ \forall t\ge\tau$ and $\hat{\pi}^{d}=\tilde{\pi}^{d}$ on $\targetset$, we have $\tilde{\state}_{\tilde t}=\hat{\state}_{\tau+\tilde t}\in\targetset\setminus\failureset,\ \forall\tilde t\ge0$.
Since $\tilde{\pi}^{d}$ was arbitrary, \autoref{def:terminal_set} gives $\state_\tau\in\stayset=\reachtargetset$.
For all $t<\tau$, $\state_t\notin\failureset\cup\targetset$ and $\interior(\targetset)\setminus\reachtargetset\subseteq\targetset$, so $\state_t\notin\reachfailureset$ by~\eqref{eq:reach_sets}.
At $t=\tau$, $\state_\tau\in\reachtargetset$ and $\state_\tau\notin\reachfailureset$ by \eqref{eq:reach_sets}.
Thus the same $\cpolicy$ and $\tau$ satisfy \autoref{def:ra_set}, proving $\state\in\raset$.
\end{proof}

In practice, $0$-level sets of reachability values often appear as boundary-like lower-dimensional sets, both in numerical solutions and learned approximations~\cite{fisac2019bridging,hsu2021safety,oh2026synthesis}.
This motivates our assumption that $\mathcal Z_{\mathrm R}$ has Lebesgue measure zero.

\begin{corollary}[Almost-Everywhere \glsentrytext{sRAS} Characterization]
\label{cor:almost_everywhere_sras_characterization}
If $\mathcal Z_{\mathrm R}$ has Lebesgue measure zero, then $\raset$ and $\srasset$ coincide almost everywhere.
\end{corollary}

\section{Robust Switched \glsentrytext{sRAS} \texorpdfstring{$Q$}{Q}-\glsentrytext{CBF}}
\label{sec:qcbf}
In this section, we show that $\reachval$ and $\stayval$ jointly form a valid robust \gls{DCBF} for the state-dependent switched system that represents the transition from the reach mode to the stay mode. 
The resulting \gls{sRAS} $Q$-\gls{CBF} safety filter, under every admissible disturbance realization, is recursively feasible and preserves membership in $\srasset$ for almost every initial state in $\srasset$ under a measure-zero condition. 
Moreover, if the filtered trajectory enters $\stayset$ in finite time, it remains there thereafter and satisfies the \gls{sRAS} specification in \autoref{def:sras_set}.

\subsection{State-Dependent Switched-System Formulation}
\label{subsec:switched_system}

We formalize the reach-to-stay transition using a state-dependent switched discrete-time control system~\cite{liberzon1999basic}.

\begin{definition}[State-Dependent Switched Discrete-Time Control System]
\label{def:switched_system}
A state-dependent switched discrete-time control system consists of a finite mode set $\modeset$, subsystem dynamics $\{\dyn_{\mode}\}_{\mode\in\modeset}$ with $\dyn_{\mode}:\xset\times\cset\times\dset\to\xset$, and a switching law $\sigma:\xset\to\modeset$.
The active mode at time $t$ is $\sigma_t:=\sigma(\state_t)$, and the state evolves as
\begin{equation}
\label{eq:switched_system}
\state_{t+1}=\dyn_{\sigma_t}(\state_t,\ctrl_t,\dstb_t).
\end{equation}
For each $\mode\in\modeset$, the operating region is $\xset_{\mode}:=\{\state\in\xset\mid\sigma(\state)=\mode\}$.
A switch occurs when $\sigma_{t+1}\neq\sigma_t$.
\end{definition}

\begin{definition}[Robust Switched \glsentrytext{DCBF}]
\label{def:switched_dcbf}
For the switched system in \autoref{def:switched_system}, let $\mathcal F_{\mode}\subseteq\xset$ and $h_{\mode}:\xset\to\reals$ denote the failure set and barrier function of each $\mode\in\modeset$, with $\Omega_{h_{\mode}}:=\{\state\in\xset_{\mode}\mid h_{\mode}(\state)\ge0\}$.
The family $\{h_{\mode}\}_{\mode\in\modeset}$ is a valid robust switched \gls{DCBF} if, for every $\mode\in\modeset$, there exists a class-$\classK$ function $\beta_{\mode}$, defined on a domain containing $h_{\mode}(\Omega_{h_{\mode}})$, such that the following conditions hold.

\noindent\textbf{(i) Mode-wise safety:}
$\Omega_{h_{\mode}}\cap\mathcal F_{\mode}=\varnothing$.

\noindent\textbf{(ii) Switched barrier condition:}
For every $\state\in\Omega_{h_{\mode}}$, there exists $\ctrl\in\cset$ such that, for every $\dstb\in\dset$, letting $\state':=\dyn_{\mode}(\state,\ctrl,\dstb)$ and $\mode':=\sigma(\state')$,
\begin{equation}
\label{eq:switched_dcbf_condition}
h_{\mode'}(\state')
\ge
\begin{cases}
\beta_{\mode}\bigl(h_{\mode}(\state)\bigr), & \mode'=\mode,\\
0, & \mode'\neq\mode.
\end{cases}
\end{equation}
\end{definition}

Condition~(ii) requires one control to satisfy the conventional robust \gls{DCBF} condition in \autoref{def:rdcbf} for same-mode successors and place switched successors in the destination certificate's $0$-superlevel set, for every disturbance.
The latter is analogous to the jump condition of hybrid \glspl{CBF}~\cite{lindemann2021learning}.

\subsection{Robust Switched \glsentrytext{sRAS} \texorpdfstring{$Q$}{Q}-\glsentrytext{CBF}}
\label{subsec:switched_sras_qcbf}

We represent the \gls{sRAS} problem with a transition from reach to stay mode as a state-dependent switched system.

\begin{definition}[\glsentrytext{sRAS} Switched System]
\label{def:sras_switched_system}
Let $\modeset:=\{\reachmode,\staymode\}$ and $\dyn_{\reachmode}=\dyn_{\staymode}:=\dyn$, and associate the reach and stay modes with failure sets $\reachfailureset$ and $\stayfailureset$, respectively.
The state-dependent switching law is
\begin{equation}
\label{eq:sras_switching_law}
\sigma(\state)
:=
\begin{cases}
\reachmode\ \text{(reach mode)}, & \state\notin\stayset,\\
\staymode\ \text{(stay mode)}, & \state\in\stayset.
\end{cases}
\end{equation}
\end{definition}

We establish feasibility of the mode-dependent $Q$-\gls{CBF} constraints and their equivalence to robust \gls{DCBF} constraints.

\begin{lemma}[Feasibility of \texorpdfstring{$Q$}{Q}-\glsentrytext{CBF} Constraints]
\label{lem:switched_q_feasibility}
Let $\beta_{\mathrm R}$ and $\beta_{\mathrm S}$ be class-$\classK$ functions satisfying $\beta_{\mathrm R}(\rho)\le\rho$ and $\beta_{\mathrm S}(\rho)\le\rho$ for all $\rho\ge0$.
For every $\state\in\raset\setminus\stayset$, there exists $\ctrl\in\cset$ such that
\begin{equation}
\label{eq:reach_q_constraint}
\min_{\dstb\in\dset}
\reachq(\state,\ctrl,\dstb)
\ge
\beta_{\mathrm R}\bigl(\reachval(\state)\bigr).
\end{equation}
For every $\state\in\stayset$, there exists $\ctrl\in\cset$ such that
\begin{equation}
\label{eq:stay_q_constraint}
\min_{\dstb\in\dset}
\stayq(\state,\ctrl,\dstb)
\ge
\beta_{\mathrm S}\bigl(\stayval(\state)\bigr).
\end{equation}
\end{lemma}

\begin{proof}
For $\state\in\raset\setminus\stayset$, choose $\ctrl=\reachfallback(\state)$.
By \eqref{eq:reach_lift} and \eqref{eq:reach_fallback}, $\min_{\dstb}\reachq(\state,\reachfallback(\state),\dstb)=\reachval(\state)\ge\beta_{\mathrm R}(\reachval(\state))$, where the inequality follows from $\reachval(\state)\ge0$ and $\beta_{\mathrm R}(\rho)\le\rho$.
The stay constraint follows identically from \eqref{eq:stay_lift}--\eqref{eq:stay_fallback}.
\end{proof}

\begin{lemma}[Equivalence of \texorpdfstring{$Q$}{Q}-\glsentrytext{CBF} Constraints]
\label{lem:switched_q_equivalence}
Let $\beta_{\mathrm R}$ and $\beta_{\mathrm S}$ be class-$\classK$ functions satisfying $\beta_{\mathrm R}(\rho)\le\rho$ and $\beta_{\mathrm S}(\rho)\le\rho$ for all $\rho\ge0$.
For every $\state\in\raset\setminus\stayset$ and $\ctrl\in\cset$, \eqref{eq:reach_q_constraint} is equivalent to
\begin{equation}
\label{eq:reach_dcbf_constraint}
\min_{\dstb\in\dset}
\reachval\bigl(\dyn(\state,\ctrl,\dstb)\bigr)
\ge
\beta_{\mathrm R}\bigl(\reachval(\state)\bigr).
\end{equation}
For every $\state\in\stayset$ and $\ctrl\in\cset$, \eqref{eq:stay_q_constraint} is equivalent to
\begin{equation}
\label{eq:stay_dcbf_constraint}
\min_{\dstb\in\dset}
\stayval\bigl(\dyn(\state,\ctrl,\dstb)\bigr)
\ge
\beta_{\mathrm S}\bigl(\stayval(\state)\bigr).
\end{equation}
\end{lemma}

\begin{proof}
Fix $\state\in\raset\setminus\stayset$, $\ctrl\in\cset$, and let $c_{\mathrm R}:=\beta_{\mathrm R}(\reachval(\state))$.
Equations \eqref{eq:ra_set_value}, \eqref{eq:reach_target_margin_set}, \eqref{eq:reach_bellman}, and $\beta_{\mathrm R}(\rho)\le\rho$ give $\reachmargin(\state)\ge\reachval(\state)\ge c_{\mathrm R}\ge0>\reachtargetmargin(\state)$.
Thus, for every $\dstb\in\dset$, \eqref{eq:reach_q} gives $\reachq(\state,\ctrl,\dstb)\ge c_{\mathrm R}$ if and only if $\reachval(\dyn(\state,\ctrl,\dstb))\ge c_{\mathrm R}$.

For $\state\in\stayset$, the stay-mode equivalence follows analogously from \eqref{eq:stay_set_value}, \eqref{eq:stay_bellman}, $\beta_{\mathrm S}(\rho)\le\rho$, and \eqref{eq:stay_q}.
\end{proof}

These results allow us to show that $\reachval$ and $\stayval$ jointly form a robust switched \gls{DCBF} for the \gls{sRAS} switched system.

\begin{theorem}[\glsentrytext{sRAS} \texorpdfstring{$Q$}{Q}-\glsentrytext{CBF}]
\label{thm:switched_sras_qcbf}
Let $\beta_{\mathrm R}$ and $\beta_{\mathrm S}$ be class-$\classK$ functions satisfying $\beta_{\mathrm R}(\rho)\le\rho$ and $\beta_{\mathrm S}(\rho)\le\rho$ for all $\rho\ge0$.
For the \gls{sRAS} switched system in \autoref{def:sras_switched_system}, let $h_{\reachmode}:=\reachval$ and $h_{\staymode}:=\stayval$.
Then the family $\{h_{\reachmode}, h_{\staymode}\}$ is a valid robust switched \gls{DCBF} in the sense of \autoref{def:switched_dcbf}.
\end{theorem}

\begin{proof}
We verify the two conditions in \autoref{def:switched_dcbf}.

\noindent\textbf{(i) Mode-wise safety.}
By \eqref{eq:sras_switching_law}, \eqref{eq:ra_set_value}, and \eqref{eq:stay_set_value}, $\Omega_{h_{\reachmode}}=\raset\setminus\stayset$ and $\Omega_{h_{\staymode}}=\stayset$.
By \autoref{def:ra_set} and \autoref{def:terminal_set}, $\raset\cap\reachfailureset=\varnothing$ and $\stayset\cap\stayfailureset=\varnothing$, establishing mode-wise safety.

\noindent\textbf{(ii) Switched barrier condition.}
For $\state\in\raset\setminus\stayset$, \autoref{lem:switched_q_feasibility} provides a control satisfying \eqref{eq:reach_q_constraint}, and \autoref{lem:switched_q_equivalence} ensures, for every $\dstb\in\dset$, $\reachval(\state')\ge\beta_{\mathrm R}(\reachval(\state))\ge0$, where $\state':=\dyn(\state,\ctrl,\dstb)$.
If $\state'\notin\stayset$, this is the same-mode condition of \eqref{eq:switched_dcbf_condition}; if $\state'\in\stayset$, \eqref{eq:stay_set_value} gives $\stayval(\state')\ge0$, satisfying the switched condition of \eqref{eq:switched_dcbf_condition}.
For $\state\in\stayset$, the two lemmas provide a control satisfying \eqref{eq:stay_q_constraint} and ensure $\stayval(\state')\ge\beta_{\mathrm S}(\stayval(\state))\ge0$ for every $\dstb\in\dset$.
Thus $\state'\in\stayset$ by \eqref{eq:stay_set_value} and the same-mode condition of \eqref{eq:switched_dcbf_condition} is satisfied.
\end{proof}

\subsection{Robust Switched \glsentrytext{sRAS} \texorpdfstring{$Q$}{Q}-\glsentrytext{CBF} Safety Filter}
\label{subsec:switched_filter}

Building on \autoref{thm:switched_sras_qcbf}, we now propose the \gls{sRAS} $Q$-\gls{CBF} safety filter for black-box systems under uncertainty.

\begin{definition}[\glsentrytext{sRAS} \texorpdfstring{$Q$}{Q}-\glsentrytext{CBF} Safety Filter]
\label{def:switched_filter}
Let $\beta_{\mathrm R}$ and $\beta_{\mathrm S}$ be class-$\classK$ functions satisfying $\beta_{\mathrm R}(\rho)\le\rho$ and $\beta_{\mathrm S}(\rho)\le\rho$ for all $\rho\ge0$, and let $\policyTask:\xset\to\cset$ be any task policy.
For each $\state\in\raset$, the \gls{sRAS} $Q$-\gls{CBF} safety filter solves
\begin{subequations}
\label{eq:switched_filter}
\begin{align}
&\safetyFilter(\state,\policyTask(\state))
\in \argmin_{\ctrl\in\cset}\left\|\ctrl-\policyTask(\state)\right\|_2^2
\label{eq:switched_filter_objective}\\
\mathrm{s.t.}\,\,
&\min_{\dstb\in\dset}\reachq(\state,\ctrl,\dstb)
\ge\beta_{\mathrm R}(\reachval(\state)),\,\,\state\in\raset\setminus\stayset
\label{eq:switched_filter_reach}\\
&\min_{\dstb\in\dset}\stayq(\state,\ctrl,\dstb)
\ge\beta_{\mathrm S}(\stayval(\state)),\,\,\state\in\stayset.
\label{eq:switched_filter_stay}
\end{align}
\end{subequations}
\end{definition}

Crucially, given $\reachval$, $\stayval$ and their lifted forms $\reachq$, $\stayq$, \eqref{eq:switched_filter_reach}--\eqref{eq:switched_filter_stay} can be evaluated for black-box systems without known dynamics, control-affinity, or value gradients.

Finally, we show that the filter is recursively feasible and preserves the ability to satisfy the \gls{sRAS} specification.

\begin{theorem}[\glsentrytext{sRAS} Feasibility Preservation]
\label{thm:recursive_feasibility}
For any task policy $\policyTask:\xset\to\cset$, let $\pi_{\mathrm F}^{u}(\state):=\safetyFilter(\state,\policyTask(\state))$ denote the filtered policy from \autoref{def:switched_filter}.
For every initial state $\state_0\in\raset\setminus\mathcal Z_{\mathrm R}\subseteq\srasset$ and every disturbance policy $\dpolicy:\xset\to\dset$, the \gls{sRAS} $Q$-\gls{CBF} safety filter is recursively feasible and satisfies $\traj_{\state_0}^{\pi_{\mathrm F}^{u},\dpolicy}(t)\in\raset\setminus\mathcal Z_{\mathrm R}\subseteq\srasset,\ \forall t\ge0$.
Moreover, if $\traj_{\state_0}^{\pi_{\mathrm F}^{u},\dpolicy}$ enters $\stayset$ at $\tau\ge0$, then $\traj_{\state_0}^{\pi_{\mathrm F}^{u},\dpolicy}(t)\in\stayset,\ \forall t\ge\tau$, and $\traj_{\state_0}^{\pi_{\mathrm F}^{u},\dpolicy}$ satisfies the \gls{sRAS} conditions in \autoref{def:sras_set}.
\end{theorem}

\begin{proof}
By \autoref{prop:sras_set_characterization}, $\raset\setminus\mathcal Z_{\mathrm R}\subseteq\srasset$.
Fix arbitrary $\state_0\in\raset\setminus\mathcal Z_{\mathrm R}$, $\policyTask$, and $\dpolicy$, and let $\state_t:=\traj_{\state_0}^{\pi_{\mathrm F}^{u},\dpolicy}(t)$.

If $\state_t\in\stayset\subseteq\raset\setminus\mathcal Z_{\mathrm R}$, \autoref{lem:switched_q_feasibility} gives feasibility of \eqref{eq:switched_filter_stay}, and \autoref{lem:switched_q_equivalence} gives $\stayval(\state_{t+1})\ge\beta_{\mathrm S}(\stayval(\state_t))\ge0$.
Thus $\state_{t+1}\in\stayset$ by \eqref{eq:stay_set_value}, and by induction $\state_{\bar t}\in\stayset,\ \forall \bar t\ge t$.

If $\state_t\in\raset\setminus(\stayset\cup\mathcal Z_{\mathrm R})$, \autoref{lem:switched_q_feasibility} gives feasibility of \eqref{eq:switched_filter_reach}, \eqref{eq:ra_set_value} and \eqref{eq:zero_level_set} give $\reachval(\state_t)>0$, and \autoref{lem:switched_q_equivalence} gives $\reachval(\state_{t+1})\ge\beta_{\mathrm R}(\reachval(\state_t))>0$. 
Hence $\state_{t+1}\in\raset\setminus\mathcal Z_{\mathrm R}$ by \eqref{eq:ra_set_value} and \eqref{eq:zero_level_set}.
Together with $\state_t\in\stayset$ case, induction yields recursive feasibility and $\state_t\in\raset\setminus\mathcal Z_{\mathrm R}\subseteq\srasset,\ \forall t\ge0$.

Finally, whenever $\state_t\in\raset\setminus(\stayset\cup\mathcal Z_{\mathrm R})$, \eqref{eq:ra_set_value} and \eqref{eq:zero_level_set} give $\reachval(\state_t)>0$, so \eqref{eq:reach_bellman} gives $\reachmargin(\state_t)>0$.
Since \eqref{eq:reach_target_margin_set} gives $\reachtargetmargin(\state_t)<0$, \eqref{eq:reach_failure_margin} implies $\margin(\state_t)>0$ and $\stagecostsafe(\state_t)<0$.
Therefore, if $\state_t$ reaches $\stayset$, it satisfies \autoref{def:sras_set}.
\end{proof}

\begin{corollary}[Almost-Everywhere Maximality]
\label{cor:almost_everywhere_maximality}
If $\mathcal Z_{\mathrm R}$ has Lebesgue measure zero, then, by \autoref{prop:sras_set_characterization} and \autoref{cor:almost_everywhere_sras_characterization}, the guarantees of \autoref{thm:recursive_feasibility} apply to almost every initial state in $\srasset$ and any task policy against every admissible disturbance policy.
\end{corollary}

\begin{corollary}[Conservative \glsentrytext{sRAS} Feasibility Preservation]
\label{cor:conservative_sRAS_certification}
Without the measure-zero condition, the filter still preserves \gls{sRAS} feasibility on $\raset\setminus\mathcal Z_{\mathrm R}\subseteq\srasset$, although this certified domain may be conservative relative to $\srasset$.
\end{corollary}

\section{Synthesis and Deployment of \glsentrytext{sRAS} \texorpdfstring{$Q$}{Q}-\glsentrytext{CBF}}
\label{sec:neural_synthesis}
The \gls{sRAS} $Q$-\gls{CBF} filter in \autoref{def:switched_filter} requires 1) access to $\reachval$, $\stayval$ and their lifted forms $\reachq$, $\stayq$, and 2) solving an inner minimization over disturbances.
We address both requirements via game-theoretic \gls{RL}~\cite{oh2026synthesis,nguyen2025gameplay,wang2024magics}: we sequentially learn the values and the corresponding fallbacks, and then train the best-response disturbances for runtime filtering.

\subsection{Sequential Value Synthesis}
\label{subsec:sequential_value_learning}
For each mode $m\in\{\mathrm S,\mathrm R\}$, we jointly train a neural critic $\mathcal Q_{\omega_m}$, a best-effort fallback $\pi^u_{\theta_m}(\cdot\mid\state)$, and an adversarial disturbance $\pi^d_{\psi_m}(\cdot\mid\state,\ctrl)$.
The disturbance retains an instantaneous informational advantage by observing and reacting to the control, following the $\max_\ctrl\min_\dstb$ game structure.
Writing the actors' deployment outputs as $\pi^u_{\theta_m}(\state)$ and $\pi^d_{\psi_m}(\state,\ctrl)$, these networks provide the value estimate
\begin{equation}
\label{eq:neural_value}
\widehat V_m(\state):=\mathcal Q_{\omega_m}\Bigl(\state,\pi^u_{\theta_m}(\state),\pi^d_{\psi_m}\bigl(\state,\pi^u_{\theta_m}(\state)\bigr)\Bigr).
\end{equation}
The learned fallback approximates \eqref{eq:stay_fallback} or \eqref{eq:reach_fallback}, while the disturbance approximates a locally worst-case response to it.

Using transitions $(\state,\ctrl,\dstb,\state')$ collected through \emph{black-box queries} of \eqref{eq:dynamics} and sampled from a replay buffer $\mathcal B_m$, we train the critic against the \emph{discounted} Isaacs target~\cite{oh2026synthesis,nguyen2025gameplay,wang2024magics}:
\begin{subequations}
\label{eq:neural_targets}
\begin{align}
\mathcal L_m^Q(\omega_m,&\theta_m,\psi_m):=\mathbb E\!\left[\left(\mathcal Q_{\omega_m}(\state,\ctrl,\dstb)-y_m\right)^2\right],\label{eq:neural_critic_loss}\\
y_{\mathrm S}&:=(1-\gamma)\staymargin(\state)+\gamma\min\{\staymargin(\state),q'_{\mathrm S}\},\label{eq:neural_stay_target}\\
y_{\mathrm R}&:=(1-\gamma)\min\{\hat\ell_{\mathrm R}(\state),\hat g_{\mathrm R}(\state)\}\notag\\
&\quad+\gamma\min\{\hat g_{\mathrm R}(\state),\max\{\hat\ell_{\mathrm R}(\state),q'_{\mathrm R}\}\},\label{eq:neural_reach_target}
\end{align}
\end{subequations}
where $\gamma\in(0,1)$ is the discount factor and $q'_m:=\mathcal Q_{\omega'_m}(\state',\ctrl',\dstb')$ is the output of a slowly updated target critic, with $\ctrl'\sim\pi^u_{\theta_m}(\cdot\mid\state')$ and $\dstb'\sim\pi^d_{\psi_m}(\cdot\mid\state',\ctrl')$.
We first train for the stay mode, then freeze $\widehat V_{\mathrm S}$ and train for the reach mode using $\hat\ell_{\mathrm R}:=\widehat V_{\mathrm S}$ and $\hat g_{\mathrm R}:=\min\{\margin,\max\{-\stagecostsafe,\hat\ell_{\mathrm R}\}\}$.

For each mode, we train the fallback to maximize the critic and the disturbance to minimize it, with entropy regularization encouraging exploration for both actors.
All updates use sampled transitions and neural network gradients, without differentiating through the dynamics.
In doing so, we adopt \gls{GDA} with \emph{finite timescale separation}: $\psi_m$ is updated on a faster timescale than $\theta_m$, allowing $\pi^d_{\psi_m}$ to track a locally best-effort worst-case response to the slowly evolving fallback $\pi^u_{\theta_m}$.
The underlying game-theoretic \gls{RL} framework guarantees local convergence to a minimax equilibrium in the policy space~\cite{wang2024magics}.

\subsection{Universal Worst-Case Response Disturbance Policy}
\label{subsec:worst_case_disturbance_synthesis}
While $\pi^d_{\psi_m}$ approximates a locally worst-case response to $\pi^u_{\theta_m}$, it need not do so for the \emph{filtered policy} from \autoref{def:switched_filter}.
For each mode, we therefore freeze $\omega_m$ and further train a copy of $\pi^d_{\psi_m}$ to minimize the critic $\mathcal Q_{\omega_m}$ over state--control pairs sampled using a family of diverse control policies.
The resulting \emph{universal best-effort worst-case response disturbance policy}, denoted by $\pi^d_{\tilde\psi_m}(\cdot\mid\state,\ctrl)$, is trained to approximate the pointwise minimizer of $\mathcal Q_{\omega_m}(\state,\ctrl,\cdot)$ across states and candidate controls, providing a learned surrogate for the inner disturbance minimization.
Using the deployment output of $\pi^d_{\tilde\psi_m}$, we replace the constraint in \autoref{def:switched_filter} with
\begin{equation}
\label{eq:neural_filter_constraint}
\mathcal Q_{\omega_m}\bigl(\state,\ctrl,\pi^d_{\tilde\psi_m}(\state,\ctrl)\bigr)\ge\beta_m\bigl(\widehat V_m(\state)\bigr).
\end{equation}
Evaluating \eqref{eq:neural_filter_constraint} requires only neural network evaluations, enabling runtime filtering on black-box systems without nested disturbance optimization, known dynamics, control- or disturbance-affine structure, or value derivatives.

\begin{remark}[Post-hoc Verification]
Guarantees in \autoref{sec:qcbf} do not automatically extend to learned values and policies.
Offline conformal prediction offers probabilistic guarantees for learned reachable sets~\cite{lin2024verification} and safety filters~\cite{hu2026permissive}.
\end{remark}

\begin{figure*}[t]
\centering
\vspace*{5pt}
\includegraphics[width=\textwidth]{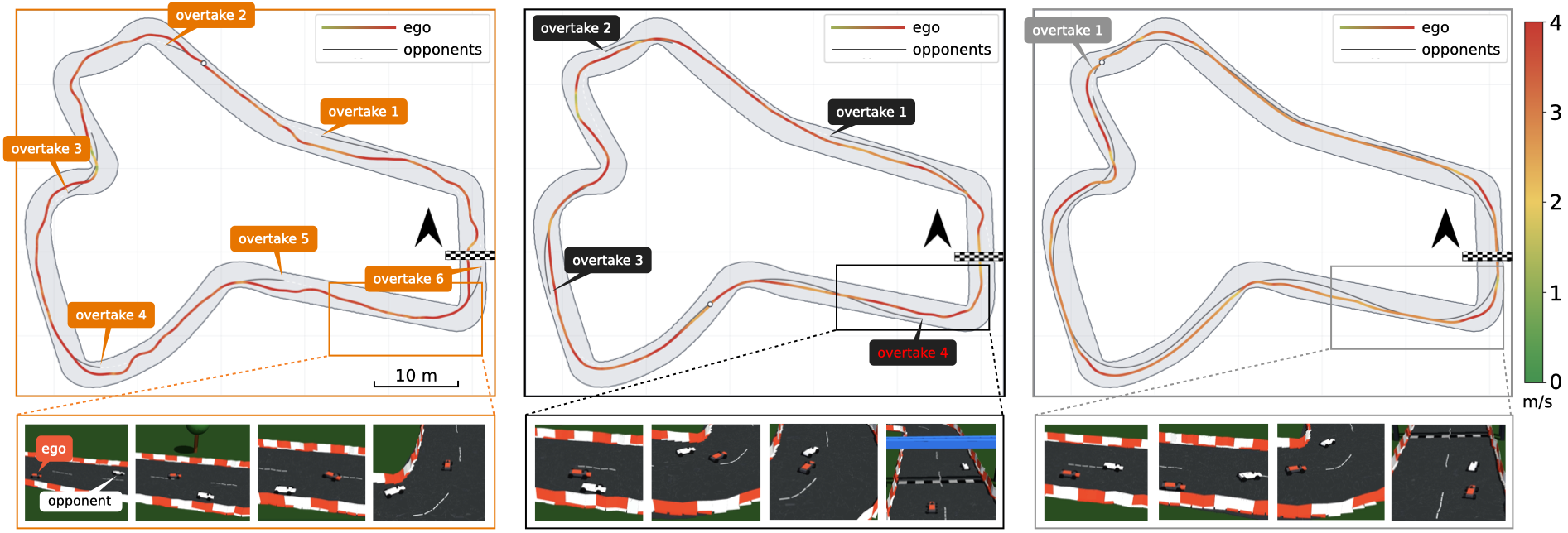}
\caption{Representative F1TENTH races, with ego trajectories colored by speed.
Our \sRAS~filter (left) maintains safety while enabling successful overtakes and lead retention, whereas \RA~(middle) can lose the lead after initially getting ahead and \AO~(right) remains blocked behind a slower opponent.}
\label{fig:f1_overtaking_circuit}
\par\medskip
\begin{minipage}[t]{\columnwidth}
\vspace{0pt}
\centering
\small
\setlength{\tabcolsep}{4pt}
\renewcommand{\arraystretch}{1.1}
\resizebox{\columnwidth}{!}{
\begin{tabular}{l|ccc}
\hline
Method & Safe Rate [\%] $\uparrow$ & Overtakes/Race $\uparrow$ & 2-Lap Time [s] $\downarrow$ \\
\hline
\AO & $\mathbf{98.2\,[97.2,\,98.9]}$ & $1.00\pm0.03$ & $111.40\pm3.32$ \\
\RA & $56.8\,[53.7,\,59.8]$ & $3.59\pm0.90$ & $116.43\pm8.74$ \\
\sRAS & $\mathbf{98.0\,[96.9,\,98.7]}^{\ddagger}$ & $\mathbf{9.63\pm0.99}^{\dagger\ddagger}$ & $\mathbf{97.37\pm2.87}^{\dagger\ddagger}$ \\
\hline
\end{tabular}
}
\par\smallskip
\refstepcounter{table}
\label{tab:f1_overtaking_results}
\noindent
\begin{minipage}{\columnwidth}
\textbf{TABLE~\Roman{table}:} F1TENTH results from $1000$ simulated races per filter.
Our \sRAS~filter achieves a $98.0\%$ safe rate with significantly more overtakes per race and shorter 2-lap times than both baselines.
Safe rates include $95\%$ Wilson confidence intervals; other metrics show mean $\pm$ standard deviation.
For \sRAS, $\dagger$ and $\ddagger$ indicate significant differences ($p<0.01$) from \AO~and \RA, respectively.
\end{minipage}
\end{minipage}\hfill%
\begin{minipage}[t]{\columnwidth}
\vspace{0pt}
\centering
\small
\setlength{\tabcolsep}{4pt}
\renewcommand{\arraystretch}{1.1}
\resizebox{\columnwidth}{!}{
\begin{tabular}{l|ccc}
\hline
Method & Safe Rate [\%] $\uparrow$ & Cross Rate [\%] $\uparrow$ & Success Rate [\%] $\uparrow$ \\
\hline
\AO & $85.6\,[83.3,\,87.6]$ & $0.8\,[0.4,\,1.6]$ & $0.7\,[0.3,\,1.4]$ \\
\RA & $63.4\,[60.4,\,66.3]$ & $87.4\,[85.2,\,89.3]$ & $63.3\,[60.3,\,66.2]$ \\
\sRAS & $\mathbf{86.9}\,[84.7,\,88.9]^{\ddagger}$ & $\mathbf{89.7}\,[87.7,\,91.4]^{\dagger}$ & $\mathbf{86.9}\,[84.7,\,88.9]^{\dagger\ddagger}$ \\
\hline
\end{tabular}
}
\par\smallskip
\refstepcounter{table}
\label{tab:go2_gap_results}
\noindent
\begin{minipage}{\columnwidth}
\textbf{TABLE~\Roman{table}:} Go2 gap-jumping results from $1000$ simulated episodes per filter.
Our \sRAS~filter achieves an $86.9\%$ success rate (percentage of episodes with safe crossing, safe landing, and continued safety), significantly higher than both baselines.
All metrics show $95\%$ Wilson confidence intervals.
For \sRAS, $\dagger$ and $\ddagger$ denote significant differences ($p<0.01$) from \AO~and \RA, respectively.
\end{minipage}
\end{minipage}
\vspace{-0.5cm}
\end{figure*}

\section{Experimental Results}
\label{sec:experiments}
In this section, we evaluate the proposed \gls{sRAS} filter on two robotic tasks: quadruped gap jumping in simulation and hardware, and simulated F1TENTH endurance racing.

\noindent \textbf{Baselines.}
We compare against Avoid-only and \gls{RA} baselines.
The Avoid-only filter uses an avoid-only value trained with safety margin $\margin$, whereas the \gls{RA} filter uses a reach--avoid value trained with safety margin $\margin$ and target margin $\stagecostsafe$.
Both baseline filters apply robust $Q$-\gls{CBF} interventions~\cite{oh2026synthesis}.
These comparisons assess the benefits of encoding target reaching and safe retention after first entry.

\subsection{F1TENTH Endurance Race}
\label{subsec:f1_experiment}

We first evaluate our \gls{sRAS} filter in a high-fidelity F1TENTH simulator without disturbances (singleton $\dset$), built in MuJoCo~\cite{todorov2012mujoco} based on \texttt{F1TENTH Gym}~\cite{o2020f1tenth}.
The ego must race for two laps, safely overtake as many opponents as possible, and minimize completion time.
An overtake counts as successful only after the ego gets ahead and retains the lead for $5\,\mathrm{s}$, after which the opponent is replaced by a new one spawned at a random location ahead of the ego.

\noindent \textbf{Setup.}
The safety margin $\margin(\state):=\min\{g_{\mathrm{wall}}(\state),g_{\mathrm{oppo}}(\state)\}$ is the minimum of the ego's signed clearance margins to the track boundary and opponent.
The target margin $\stagecostsafe(\state):=CP_{\mathrm{ego}}-CP_{\mathrm{opp}}-0.58\,\mathrm m$ is the difference in unwrapped centerline progress in meters, with a one-vehicle-length lead buffer.
We use a fixed model predictive path integral (MPPI) policy~\cite{williams2017model} as the task policy for both ego and opponent, with an additional filter applied to the ego.
The common $90$-dimensional base observation includes normalized LiDAR measurements, ego velocity and yaw rate, previous controls, and ego-centric opponent features.
The \gls{sRAS} filter additionally observes the normalized target margin in both modes and the frozen stay value in reach mode.
The control inputs are commanded steering angle and speed.

\noindent \textbf{Evaluation and Metrics.}
We evaluate all filters on the same $1000$ random seeds.
After a collision, the ego is held at the impact position for $3\,\mathrm{s}$ before respawning at the nearest track centerline point.
We report \textit{safe rate} (percentage of races completing two laps without collision), \textit{mean successful overtakes per race}, and \textit{mean two-lap completion time}.

\noindent \textbf{Results.}
Table~\ref{tab:f1_overtaking_results} reports the results, with significance assessed using exact McNemar tests for safe rate and paired $t$-tests for other metrics (Holm-adjusted $p<0.01$).
The \gls{sRAS} filter achieves a high safe rate similar to Avoid-only and significantly higher than \gls{RA}.
It also achieves significantly more successful overtakes and significantly shorter two-lap completion times than both baselines.
As Figure~\ref{fig:f1_overtaking_circuit} illustrates, \gls{sRAS} repeatedly passes opponents and retains the lead, whereas Avoid-only does not encode passing and remains blocked behind a slower opponent.
\gls{RA} does not require safe lead retention, so it can collide after briefly getting ahead or lose the lead and enter battles that slow down both vehicles.

\subsection{Quadruped Gap Jumping}
\label{subsec:go2_experiment}

We next evaluate the \gls{sRAS} filter on a $36$-D Unitree Go2 quadruped in hardware and in MuJoCo simulation~\cite{todorov2012mujoco}.
The robot must jump across a $25\,\mathrm{cm}$ gap, land safely on the far side, and remain safe afterward.
To test robustness against adversarial uncertainty realizations in simulation, we apply external forces up to $30\,\mathrm{N}$ in arbitrary directions at arbitrary torso locations, selected by the disturbance policy trained following \autoref{subsec:worst_case_disturbance_synthesis}.
Accounting for these disturbances also aims to mitigate the sim-to-real gap.

\noindent \textbf{Setup.}
The safety margin $\margin(\state)$ is the minimum of normalized margins for trunk clearance above local terrain, body tilt, and non-foot contact force.
The target margin $\stagecostsafe(\state)$ combines progress beyond the gap, foot-to-terrain proximity, uprightness, and a planar-speed margin modulated by foot proximity to encode an upright arrival near the ground on the far side.
In both simulation and hardware, we use a pure-pursuit walking task policy toward the far side of the gap.
The observation consists of a 5-frame history of body-frame angular velocity, projected gravity vector, velocity command, gait phase, joint positions and velocities, and a $17\times11$ local height scan.
The control inputs are $12$ joint-position targets relative to a nominal standing pose, tracked by a PD controller at $50\,\mathrm{Hz}$.

\noindent \textbf{Evaluation and Metrics.}
We evaluate each filter in $1000$ simulation episodes.
We report \textit{safe rate} (percentage avoiding failure), \textit{cross rate} (percentage reaching the far side, regardless of landing safety), and \textit{success rate} (percentage crossing and landing safely, then remaining safe until episode end).

\noindent \textbf{Simulation Results.}
Table~\ref{tab:go2_gap_results} reports simulation results with significance assessed using exact McNemar tests (Holm-adjusted $p<0.01$).
The Avoid-only filter does not encode crossing and rarely completes the jump.
Although \gls{RA} often reaches the far side, it does not require continued safety after target entry, leaving the robot vulnerable to failure during landing or subsequent stabilization.
Figure~\ref{fig:go2_gap_hardware} illustrates such failure modes.
Our \gls{sRAS} filter modifies actions from the walking policy to preserve the feasibility of safe crossing, landing, and continued safety afterward, and achieves a significantly higher success rate than both baselines.

\noindent \textbf{Hardware Results.}
Figure~\ref{fig:go2_gap_hardware} (top) shows the \gls{sRAS}-filtered Go2 reproducing the successful behavior observed in simulation: crossing, landing safely, and remaining safe afterward.

\section{Conclusion}
We presented an \gls{sRAS} $Q$-\gls{CBF} safety filter for high-dimensional black-box systems under bounded uncertainty.
For exact values and under a measure-zero condition, the filter is recursively feasible, preserves \gls{sRAS} feasibility from almost every winnable initial state, and keeps the system safely within the target after first entry against all admissible uncertainty realizations.
Reachability-based adversarial \gls{RL} and $Q$-function lifting enable scalable synthesis and deployment without known dynamics, affine structure, value derivatives, or hand-designed barriers.
We validated our method on two robotic tasks: quadruped gap jumping in simulation and hardware, and simulated F1TENTH racing.
The filter enabled safe crossing, landing, and continued safety for the quadruped, and safe overtaking and lead retention for the F1TENTH race, demonstrating the intended \gls{sRAS} behavior.

\section*{Acknowledgment}
Claude (Anthropic) and Codex (OpenAI) were used to generate code for the experiments in \autoref{sec:experiments}.

\printbibliography

\end{document}